\documentclass[runningheads]{llncs}
\usepackage{graphicx}
\usepackage{comment}
\usepackage{amsmath,amssymb} 
\usepackage{multirow}
\usepackage{color}
\usepackage{capt-of}

\begin{document}
\pagestyle{headings}
\mainmatter
\def\ECCVSubNumber{1227}  

\title{Learning Disentangled Representations with Latent Variation Predictability} 

\titlerunning{Variation Predictability for Disentanglement}
%
\author{Xinqi Zhu \and Chang Xu \and Dacheng Tao}
\authorrunning{X. Zhu et al.}
%
\institute{UBTECH Sydney AI Centre, School of Computer Science, Faculty of Engineering, \\
    The University of Sydney, Darlington, NSW 2008, Australia\\
    \email{\{xzhu7491@uni.,c.xu@,dacheng.tao@\}sydney.edu.au}}
\maketitle

\begin{abstract}
    Latent traversal is a popular approach to visualize the
    disentangled latent representations. Given a bunch of variations
    in a single unit of the latent representation,
    it is expected that there is a change in a single
    factor of variation of the data while others are fixed.
    However, this impressive experimental observation is
    rarely explicitly encoded in the objective function of
    learning disentangled representations. This paper defines
    the \emph{variation predictability} of latent disentangled
    representations. Given image pairs generated by latent codes varying
    in a single dimension, this varied dimension could be closely
    correlated with these image pairs if the representation
    is well disentangled. Within an adversarial generation process,
    we encourage variation predictability by maximizing the
    mutual information between latent variations and corresponding
    image pairs. We further develop an evaluation metric that
    does not rely on the ground-truth generative factors to measure
    the disentanglement of latent representations.
    The proposed variation predictability is a general constraint
    that is applicable to the VAE and GAN frameworks for
    boosting disentanglement of latent representations.
    Experiments show that the proposed variation
    predictability correlates well with existing ground-truth-required
    metrics and the proposed algorithm
    is effective for disentanglement learning.

\end{abstract}

\section{Introduction}
\label{sec:introduction}
Nowadays learning interpretable representations from
high-dimensional data is of central importance for
downstream tasks such as classification
\cite{Chen2016InfoGANIR,Dupont2018LearningDJ,Jeong2019LearningDA},
domain adaptation \cite{Peng2019DomainAL,Yang2019UnsupervisedDA,SutMilSchBau19},
fair machine learing \cite{Creager2019FlexiblyFR,Locatello2019OnTF},
and reasoning \cite{Steenkiste2019AreDR}.
To achieve this goal, a series of work have been conducted
\cite{Chen2016InfoGANIR,Higgins2017betaVAELB,Kim2018DisentanglingBF,Dupont2018LearningDJ,chen2018isolating,Jeong2019LearningDA}
under the subject of \emph{disentangled representation learning}.
Although there is not a widely adopted mathmatical definition for
disentanglement, a conceptually agreed definition can be expressed
as that each unit of a disentangled representation should
capture a single interpretable variation of the data
\cite{Bengio2012RepresentationLA}.

\begin{figure}[t]
\begin{center}
   \includegraphics[width=\linewidth]{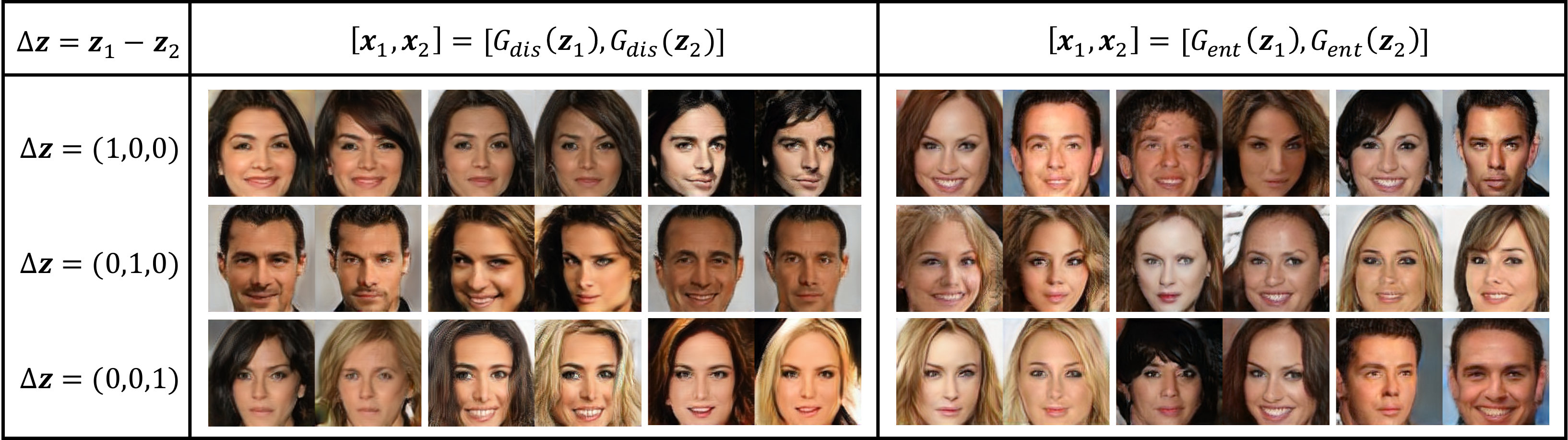}
\end{center}
    \caption{This is a table showing generated image pairs
    $[\boldsymbol{x}_{1}, \boldsymbol{x}_{2}]$ whose latent codes
    $[\boldsymbol{z}_{1}, \boldsymbol{z}_{2}]$ have difference
    in a single dimension
    $\Delta \boldsymbol{z} = \boldsymbol{z}_{1} - \boldsymbol{z}_{2}$.
    Left: image pairs generated by a disentangled generator $G_{dis}$.
    Right: image pairs generated by an entangled generator $G_{ent}$.
    In each row, the latent code difference $\Delta \boldsymbol{z}$ is
    kept fixed with only one dimention modified.
    For the disentangled image pairs (left),
    it is not difficult to tell that each row
    represents the semantics of \emph{fringe}, \emph{smile}, and
    \emph{hair color} respectively.
    However, for entangled ones (right)
    the semantics are not clear although the image pairs are
    also generated with a single dimension varied in the latent codes
    just like the left ones.}
\label{fig:intuition}
\end{figure}

One line of current most promising disentanglement methods derives from
$\beta$-VAE \cite{Higgins2017betaVAELB}, with its
variants such as FactorVAE \cite{Kim2018DisentanglingBF} and
$\beta$-TCVAE \cite{chen2018isolating} developed later.
This series of works mainly
realize disentanglement by enforcing the
independence in the latent variables. Although independence assumption
is an effective proxy for the learning of disentanglement,
this assumption can be unrealistic for real-world data as the
underlying distribution of the semantic factors may not be factorizable.
Additionally since these models are defined based upon
the Variational Autoencoder (VAE) framework \cite{Kingma2013AutoEncodingVB}
which intrinsically causes blurriness in the generated data,
their applications are mostly limited to synthetic data and real-world
data of small sizes. Another line of work to achieve
disentanglement purpose is by using
InfoGAN \cite{Chen2016InfoGANIR}, which encourages disentanglement
through maximizing the mutual information between the generated data
and a subset of latent codes. This model inherits merits from
GAN \cite{Goodfellow2014GenerativeAN} so that sharper and more realistic
images can be synthesized thus can be applied to more complex datasets.
However, we show in the experiments that
the disentanglement performance of InfoGAN is limited
and the training of InfoGAN is less stable compared with our
proposed GAN-based models.
A problem in existing disentanglement learning community is the
lack of evaluation metrics that can give a quantitative
measurement of the performance of disentanglement. Existing
metrics depends on the existence of ground-truth
generative factors and an encoder network
\cite{Higgins2017betaVAELB,Kim2018DisentanglingBF,chen2018isolating,Ridgeway2018LearningDD,Eastwood2018AFF,SutMilSchBau19},
thus the quantitative measurements are usually done on
synthetic data with VAE framework,
leaving latent traversal inspection by human
the only evaluation method for experiments on real-world data,
and this to an extent discourages the development of GAN-based models,
which are known to be effective in photorealistic image synthesis,
from disentangled representation learning.

Different from the existing disentanglement learning methods, we
reconsider the problem from the perspective of \emph{Variation}.
We argue that disentanglement can be natually described by
the correlated variations between the latent codes
and the observations, and that is
why researchers initially use latent traversals as a method to evaluate
whether a representation is disentangled.
This intuition is based on an assumption
that determining the semantics of a dimension in
disentangled representations is easy but
in entangled representations is difficult.
In Fig. \ref{fig:intuition} we show an example of this
interesting phenomenon.
On the left and right parts of Fig. \ref{fig:intuition},
there are image pairs generated by two models
(one disentangled and one entangled).
All these image pairs are generated by
varying a single dimension in the latent codes, and the varied dimension
is the same for each row (the varied dimension is indicated
by $\Delta \boldsymbol{z}$ on the left by onehot representation).
However, as there can be multiple
different latent code pairs [$\boldsymbol{z}_{1}$, $\boldsymbol{z}_{2}$] that cause
the same $\Delta \boldsymbol{z}$, the generated image pairs
can be diverse.
In Fig. \ref{fig:intuition}, it is not difficult to
tell from the image pairs that the rows
for the left model control the semantics of
\emph{fringe}, \emph{smile}, and \emph{hair color} respectively,
while for the right model the semantics are not clear.
This \emph{easy-to-tell} property
for the left model is due to the
consistent pattern in the shown image pairs,
while there is not a clear pattern for the right model
so we cannot tell what each row is controlling.

This phenomenon for distinguishing disentangled and entangled representations
motivates us to model the latent \emph{variation predictability}
as a proxy to achieve disentanglement, which is defined based on
the difficulty of predicting the varied dimension in latent codes from
the corresponding image pairs.
Taking the case in Fig. \ref{fig:intuition} as an example,
we say the left model has a higher latent
variation predictability than the right one, which corresponds
with the fact that we only need very small number of image pairs to
tell what semantics each row is controlling for the left model.
Note that rows in the entangled model (right one) may also contain stable
patterns (may not correspond to a specific semantics),
but it is difficult for us to tell what they are
from such few images.
By exploiting the variation predictability, our contributions
in this paper can be summarized as follows:
\begin{itemize}
    \item By regarding the variation predictability as a proxy for
        achieving disentanglement, we design a
        new objective which is general and effective
        for learning disentangled representations.
    \item Based on the definition of variation predictability, we propose a
        new disentanglement evaluation metric for quantitatively measuring
        the performance of disentanglement models, without
        the requirement of the ground-truth generative factors.
    \item Experiments on various datasets are conducted
        to show the effectiveness of our proposed metric and
        models. 
\end{itemize}

\section{Related Work}
\label{sec:related_work}
\textbf{Generative Adversarial Networks.} Since the introduction of
the initial GAN by Goodfellow et. al. \cite{Goodfellow2014GenerativeAN},
the performance of GANs have been thoroughly improved from
various perspectives, e.g.
generative quality
\cite{radford2015unsupervised,Karras2017ProgressiveGO,Han2019SAGAN,Brock2018LargeSG,Karras2020ASG,Karras2019AnalyzingAI},
and training stability
\cite{Arjovsky2017WassersteinGA,berthelot2017began,Gulrajani2017ImprovedTO,Kodali2018OnCA,Miyato2018SpectralNF,Karras2019AnalyzingAI}.
However, the study of semantics learned in the latent space is
less exploited.
Chen et al. \cite{Chen2016InfoGANIR}
propose InfoGAN which successfully learns disentangled representations
by maximizing the mutual information
between a subset of latent variables and the generated sampels.
Donahue et al. \cite{Donahue2017SemanticallyDT} introduce
Semantically Decomposed GANs which encourage a specified portion
of the latent space to correspond to a known source of variation, resulting
in the decomposition of identity and other contingent aspects of observations.
In \cite{Tran2017DisentangledRL}, Tran et. al.
introduce DR-GAN containing an encoder-decoder
structure that can disentangle the pose information in face
synthesis.
Karras et al. \cite{Karras2020ASG} introduce an intermediate latent space
derived by non-linear mapping from the original latent space and
conduct disentanglement study on this space with perceptual path length
and linear separability.
In \cite{NguyenPhuoc2019HoloGANUL}, HoloGANs are introduced by combining
inductive bias about the 3D world with GANs for learning a
disentangled representation of pose, shape, and appearance.
Recently, Lin et al. introduce a contrastive regularization to 
boost InfoGAN for disentanglement learning, and also propose a 
ModelCentrality scheme to achieve unsupervised model selection 
\cite{LinInfoGANCR}.
By looking into well-trained GANs, Bau et al. conducts
GAN dissection in \cite{bau2019gandissect}. They show
that neurals in GANs actually learn interpretable concepts,
and can be used for modifying contents in the generated images.
In a similar spirit, Shen et al. \cite{Shen2019InterpretingTL} introduces
InterFaceGAN showing GANs spontaneously learn various latent
subspaces corresponding to specific interpretable attributes.
Besides the existing works,
unsupervisedly learning disentangled representations with GANs and
quantifying their performance of disentanglement are
still unsolved problems.

\textbf{Variational Autoencoders and Variants.}
There are more systematic works of disentanglement learning
based on the VAE framework
\cite{Kingma2013AutoEncodingVB,Burgess2018UnderstandingDI},
and the most common method for approaching disentanglement is
by modeling the independence in the latent space.
As an early attempt of extending VAEs for
learning independent latent factors, Higgins et al. \cite{Higgins2017betaVAELB}
pointed out that modulating the
KL-term
in the learning objective of VAEs (known as evidence lower bound (ELBO))
with a single hyper-parameter $\beta > 1$
can encourage the model to learn independent latent factors:
\begin{align}
    \mathcal{L}(\theta, \phi; \boldsymbol{x}, \boldsymbol{z}, \beta)
    = \mathbb{E}_{q_{\phi}(\boldsymbol{z} | \boldsymbol{x})}
    [\text{log}\,p_{\theta}(\boldsymbol{x} | \boldsymbol{z})]
    - \beta D_{KL}(q_{\phi}(\boldsymbol{z} | \boldsymbol{x}) ||\,
    p(\boldsymbol{z}))
    \label{eq:vae}.
\end{align}
Later based on the decomposition of the KL term
$D_{KL}(q_{\phi}(\boldsymbol{z} | \boldsymbol{x}) ||\,p(\boldsymbol{z}))$
\cite{hoffman2016elbo,Makhzani2017PixelGANA},
it has been discovered that the KL divergence between
the aggregated posterior
$q_{\phi}(\boldsymbol{z}) \equiv
\mathbb{E}_{p_{data}(\boldsymbol{x})}
[q_{\phi}(\boldsymbol{z}|\boldsymbol{x})]$ and its factorial distribution
$\text{KL}(q_{\phi}(z)||\prod_{j}q_{\phi}(z_{j}))$
(known as the total correlation (TC) of the latent variables)
contributes most to the disentanglement purpose,
leading to the emergence of models
enforcing penalty on this term.
Kim et al. \cite{Kim2018DisentanglingBF} introduce FactorVAE to
minimize this TC term
through adopting a discriminator in the latent space
\cite{Goodfellow2014GenerativeAN,Makhzani2015AdversarialA}
with density-ratio trick \cite{Nguyen2010EstimatingDF,Banerjee2004ClusteringWB}.
Chen et al. \cite{chen2018isolating} introduce $\beta$-TCVAE to
employ a mini-bath weighted sampling for the TC estimation.
Kumar et al. \cite{Kumar2017VariationalIO} use moment matching
to penalize the divergence between aggregated posterior and the prior.
These works have been shown effective for disentangled representation
learning, especially after the introduction of various quantitative disentanglement metric
\cite{Higgins2017betaVAELB,Kim2018DisentanglingBF,chen2018isolating,Ridgeway2018LearningDD,Eastwood2018AFF,SutMilSchBau19}.
Dupont \cite{Dupont2018LearningDJ} introduces JointVAE for learning
continuous and discrete latent factors under the VAE framework for
stable training and a principled inference network.
Later Jeong et al. \cite{Jeong2019LearningDA} introduce CascadeVAE
which handles independence enforcing through information cascading and
solves discrete latent codes learning by an alternating
minimization scheme.
However, these works model disentanglement only from the
independence perspective, which may not be practical for real-world data.

\section{Methods}
\label{sec:method}
In this section, we first introduce the Variation Predictability objective
in section \ref{sec:vp_loss},
a general constraint which encourages disentanglement from the
perspective of predictability for the latent variations, and also
show its intergration with the GAN framework.
We then introduce the proposed Variation Predictability Evaluation Metric
in section \ref{sec:vp_metric}, a general evaluation method quantifying
the performance of disentanglement without relying on the ground-truth
generative factors. Our code is available at: 
\texttt{https://github.com/zhuxinqimac/stylegan2vp}.

\subsection{Variation Predictability Loss}
\label{sec:vp_loss}
We first introduce the concept of \emph{variation predictability}, defined
as follows:
\begin{definition}
    \label{def:1}


    If image pairs are generated by varying a single dimension in
    the latent codes, the \emph{variation predictability} represents how easy
    the prediction of the varied latent dimension from the image pairs is.
\end{definition}

Here we also need to define what \emph{easy} is in this context:
    a prediction is called easy if the required number 
    of training examples is small.

A high variation predictability means predicting the varied latent dimension
from the image pairs is easy, i.e. only a small number of
training image pairs are needed to identify this dimension
(consider the left part of Fig. \ref{fig:intuition}),
while a low predictability means the prediction is hard
(consider the right part of Fig. \ref{fig:intuition}).
As this concept corresponds well with our determination of
disentanglement (the left model in Fig. \ref{fig:intuition} is
disentangled while the right one is not),
we propose to utilize it as a proxy for
achieving the disentanglement objective.
Note that traditionally the assumption of
\emph{independence} in latent variables works as another proxy for the
learning of disentangled representation
\cite{Higgins2017betaVAELB,Kim2018DisentanglingBF,Dupont2018LearningDJ,Kumar2017VariationalIO,chen2018isolating,Jeong2019LearningDA},
which we think is a stronger assumption than
the proposed variation predictability
and it may not hold for real-world data.

In order to model the variation predictability from observations
to the varied latent dimension,
we adopt a straightforward implementation
by directly maximizing the mutual information between
the varied latent dimension $d$ and the
paired images ($\boldsymbol{x}_{1}$, $\boldsymbol{x}_{2}$)
derived by varying dimension $d$ in latent codes:
$I(\boldsymbol{x}_{1}, \boldsymbol{x}_{2}; d)$, where we name
this mutual information as the Variation
Predictability objective (VP objective) and
the negative of this term as the Variation Predictability loss (VP loss).
We instantiate our VP objective by intergrating it within the generative
adversarial network (GAN) framework.

\begin{figure}[t]
\begin{center}
   \includegraphics[width=\linewidth]{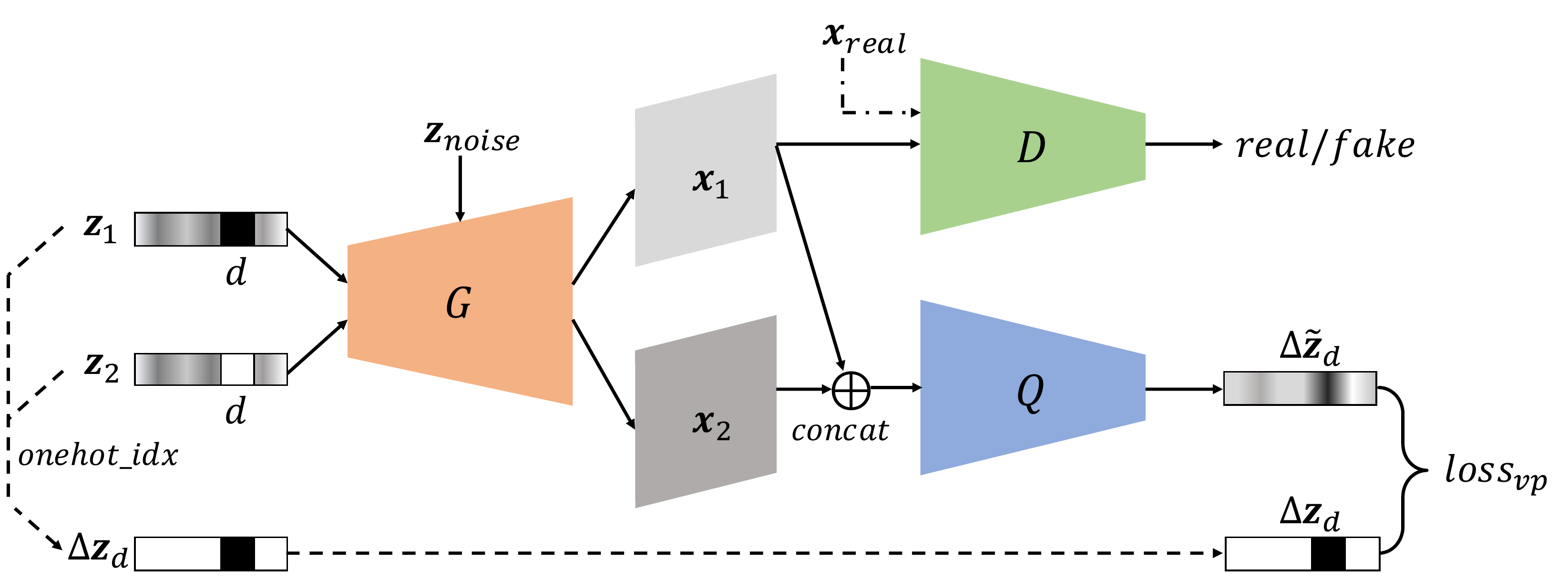}
\end{center}
    \caption{Overall architecture of our proposed model.
    The model first samples two latent codes $\boldsymbol{z}_{1}$
    and $\boldsymbol{z}_{2}$ from the latent space that
    only differs in a single dimension $d$, and this index $d$
    also serves as the target for the recognition network $Q$
    after being transformed into onehot representation
    $\Delta \boldsymbol{z}_{d}$.
    Two images are then generated with the generator
    $\boldsymbol{x}_{1} = G(\boldsymbol{z}_{1})$,
    $\boldsymbol{x}_{2} = G(\boldsymbol{z}_{2})$. $\boldsymbol{x}_{1}$ is
    fed to the discriminator to get the \emph{real/fake}
    learning signal same as in the
    standard GANs. We then concatenate both images along the channel
    axis and feed them to the recognition network $Q$ for the
    prediction of $d$.}
\label{fig:architecture}
\end{figure}

As for a brief introduction,
GAN is a generative model introduced by Goodfellow et al.
\cite{Goodfellow2014GenerativeAN} for learning the distribution of
training data. A generator network $G$ is trained to map a random noise
$\boldsymbol{z}$ to an output image $\boldsymbol{x}$:
$G(\boldsymbol{z}): \boldsymbol{z} \rightarrow \boldsymbol{x}$.
At the meantime, a discriminator network $D$ is trained to
predict whether the generated image $\boldsymbol{x}$
matches the real data distribution $p_{data}(\boldsymbol{x})$.
The training strategy of $G$ and $D$ follows a minimax game where
$G$ tries to minimize while $D$ tries to
maximize the following objective:
\begin{align}
    \underset{G}{\text{min}}\, \underset{D}{\text{max}}\, V(G, D) =
    \mathbb{E}_{\boldsymbol{x} \sim p_{data}}[\text{log}\,
    D(\boldsymbol{x})] + \mathbb{E}_{\boldsymbol{z} \sim p_{noise}}[\text{log}\,
    (1 - D(G(\boldsymbol{z})))
    ] \label{eq:gan}.
\end{align}
After convergence, the generator $G$ can
generate images look similar to the real data, while
the discriminator $D$ cannot tell whether a generated image is real or not.

The VP loss can be easily integrated into a GAN as:
\begin{align}
    \underset{G}{\text{min}}\, \underset{D}{\text{max}}\, V_{vp}(G, D) =
    V(G, D) - \alpha I(\boldsymbol{x}_{1}, \boldsymbol{x}_{2}; d)
    \label{eq:obj_with_gan},
\end{align}
where $\alpha$ is a hyper-parameter to balance the generative
quality and disentanglement quality. This model is
named as Variation Predictability GAN (VPGAN).
By optimizing the VP objective, the generator is encouraged to
synthesize images having a strong correspondance with the latent
codes so that the changes in data controlled by each latent
dimension is distinguishable from each other,
and the changes controlled by a single dimension is consistent all the time.
In other words, after the training converges
the variations in data should be
natually grouped into as many classes as the number of dimensions in
the latent space. 
In Appendix 6 we show the relation
between our objective and the InfoGAN \cite{Chen2016InfoGANIR} objective.

For the VP objective to be optimizable, we adopt the
variational lower bound of mutual information
defined as follows.
\begin{lemma}
    \label{lemma:1}
    For the mutual information between two random variables
    $I(\boldsymbol{x}; \boldsymbol{y})$, the following lower bound holds:
    \begin{align}
        I(\boldsymbol{x}; \boldsymbol{y}) \geq H(\boldsymbol{y}) +
        \mathbb{E}_{p(\boldsymbol{x},
        \boldsymbol{y})} \textnormal{log}\, q(\boldsymbol{y}|\boldsymbol{x}),
    \end{align}
    where the bound is tight when $q(\boldsymbol{y}|\boldsymbol{x}) =
    p(\boldsymbol{y}|\boldsymbol{x})$.
\end{lemma}

\begin{proof}
    See Appendix 7.
    \hfill $\square$
\end{proof}

Based on Lemma \ref{lemma:1}, we can get the lower bound of our VP objective:
\begin{align}
    \mathcal{L}_{vp}(\theta, \phi; &\boldsymbol{x}_{1}, \boldsymbol{x}_{2}, d) \\
    &=
    \mathbb{E}_{\boldsymbol{x}_{1}, \boldsymbol{x}_{2}, d
    \sim p_{\theta}(\boldsymbol{x}_{1}, \boldsymbol{x}_{2}, d)}
    \text{log}\,
    q_{\phi}(d\,|\boldsymbol{x}_{1}, \boldsymbol{x}_{2})
    + H(d) \label{eq:obj} \\
    &\leq I(\boldsymbol{x}_{1}, \boldsymbol{x}_{2}; d).
\end{align}

For the sampling of $d$, $\boldsymbol{x}_{1}$, $\boldsymbol{x}_{2}$,
we first sample the dimension index $d$ out of the number of
continuous latent variables (in this paper we focus on continuous
latent codes and leave the modeling for discrete latent codes for future works),
then we sample a latent code and sample twice on dimension $d$ so
we get a paired latent codes $[\boldsymbol{z}_{1}, \boldsymbol{z}_{2}]$
differing only on dimension $d$.
The images $\boldsymbol{x}_{1}$ and $\boldsymbol{x}_{2}$ are generated by
$G(\boldsymbol{z}_{1})$ and $G(\boldsymbol{z}_{2})$.
The conditional distribution
$q_{\phi}(d\,|\boldsymbol{x}_{1} \boldsymbol{x}_{2})$
in Eq. \ref{eq:obj} is modeled as a recognizor $Q$.
This recognizor network takes the
concatenation of the two generated images as inputs to
predict the varied dimension $d$ in the latent codes.
The architecture of our model is shown in Fig. \ref{fig:architecture}.

\subsection{Variation Predictability Disentanglement Metric}
\label{sec:vp_metric}
As the variation predictability natually defines a way to
distinguish disentangled representations and entangled representations
by the difficulty of predicting the varied latent dimensions, we can propose
a method to quantitatively measure the performance of
disentanglement once we quantify the difficulty of the prediction.
In this paper we quantify this difficulty by the performance of
doing few-shot learning \cite{Wang2019GeneralizingFA}. The intuition is that
a prediction can be seen as easy if
only a small number of training examples are needed for the prediction.
From another viewpoint, only requiring a small number of training examples
means the representation can generalize well for the prediction
task, which is also a property of \emph{disentanglement}, so this modeling
is also consistent with disentanglement itself.
We name our proposed metric as Variation Predictability metric
(VP metric), which is defined as follows:
\begin{enumerate}
    \item For a generative model, sample $N$ indices denoting which dimension
        to modify in the latent codes: $\{d^{1}, d^{2}, ..., d^{N}\}$.
    \item Sample $N$ pairs of latent codes that each pair
        only differs in the dimension sampled by step 1:
        $\{[\boldsymbol{z}^{1}_{1}, \boldsymbol{z}^{1}_{2}],
        [\boldsymbol{z}^{2}_{1}, \boldsymbol{z}^{2}_{2}], ...,
        [\boldsymbol{z}^{N}_{1},
        \boldsymbol{z}^{N}_{2}]\, | \,
        \text{Dim}_{\neq 0}(\boldsymbol{z}^{i}_{1} - \boldsymbol{z}^{i}_{2}) =
        d^{i}\}$.
    \item For each latent code pair
        $[\boldsymbol{z}^{i}_{1}, \boldsymbol{z}^{i}_{2}]$,
        generate the corresponding image pair
        $[\boldsymbol{x}^{i}_{1} = G(\boldsymbol{z}^{i}_{1}),
        \boldsymbol{x}^{i}_{2} = G(\boldsymbol{z}^{i}_{2})]$
        and their difference
        $\Delta \boldsymbol{x}^{i} = \boldsymbol{x}^{i}_{1} -
        \boldsymbol{x}^{i}_{2}$. This forms a dataset
        $\{(\Delta \boldsymbol{x}^{1}, d^{1}),
        (\Delta \boldsymbol{x}^{2}, d^{2}), ...,
        (\Delta \boldsymbol{x}^{N}, d^{N})
        \}$ with the difference of image pairs as inputs
        and the varied dimension as labels.
    \item Randomly divide the dataset into a training set and a test set
        with example numbers $\eta N$ and $(1-\eta)N$ respectively,
        where $\eta$ is the ratio of training set.
    \item Train a recognition network taking
        $\Delta \boldsymbol{x}^{i}$ as input for predicting $d^{i}$ on the
        training set. Report the accuracy $acc_{s}$ on test set.
    \item Repeat step 1 to step 5 $S$ times to get accuracies
        $\{acc_{1}, acc_{2}, ..., acc_{S}\}$.
    \item Take the average of the accuracies
        $Score_{dis} = \frac{1}{S}\sum_{s=1}^{S}acc_{s}$ as the
        disentanglement score for the model. Higher is better.
\end{enumerate}

The training set ratio $\eta$ should not be large since we need to keep
this setting as a few-shot learning so that the prediction is
sufficiently difficult to distinguish disentangled representations
and entangled representations. The reason we use the data difference
$\Delta \boldsymbol{x}$ as inputs is that this enforces the model to focus
on difference features in data caused by the varied dimension rather than
contents in the images, and we also found this implementation can achieve
a better correlation with other metrics like the FactorVAE score.
In our experiments, we choose $N = 10,000$ and $S = 3$.
For Dsprites and 3DShapes datasets we choose $\eta = 0.01$ and for CelebA
and 3DChairs datasets we choose $\eta = 0.1$ as there are more dimensions
in latent codes used for CelebA and 3DChairs. The main differences between
our proposed metric and other disentanglement metrics
\cite{Higgins2017betaVAELB,Kim2018DisentanglingBF,chen2018isolating,Ridgeway2018LearningDD,Eastwood2018AFF,SutMilSchBau19}
are that ours does not
require the existence of ground-truth generative factors
(which is not available in real-world data), and our metric does not
require an encoder so GAN-based models can also be evaluated.

\subsection{Implementations}
\label{sec:implementations}
As discovered by the recent works of StyleGANs
\cite{Karras2020ASG,Karras2019AnalyzingAI} that the latent codes
of a generator can be treated as style codes modifying a learned
constant for achieving a higher-quality generation and stabler training
than tranditional GANs, we adopt a similar strategy
to ease the training procedure of GANs in our experiments.
However, unlike StyleGANs which take a side mapping network to transform
the input codes into multiple intermediate latent codes, we directly feed the
latent codes sampled from the prior distribution into the generator network
without a mapping network to learn
the disentangled representation in the prior latent space.
The network architectures and parameters for different experiments are
shown in Appendix 10.

\section{Experiments}
\label{sec:experiments}
We first conduct experiments on popular disentanglement evaluation datasets
Dsprites \cite{dsprites17} and 3DShapes \cite{Kim2018DisentanglingBF}
to validate the effectiveness
of our proposed VP disentanglement metric. Second we evaluate our method
with VAE framework to show its complementarity to independence modeling
for achieving disentanglement. Then we equip our models with GAN framework
to conduct experiments on datasets without ground-truth generative
factors 3DChairs \cite{Aubry2014Seeing3C} and CelebA \cite{Liu2014DeepLF},
and use our proposed VP metric to quantitatively evaluate the disentanglement
performance of our models.

\begin{figure}[t]
\begin{center}
   \includegraphics[width=\linewidth]{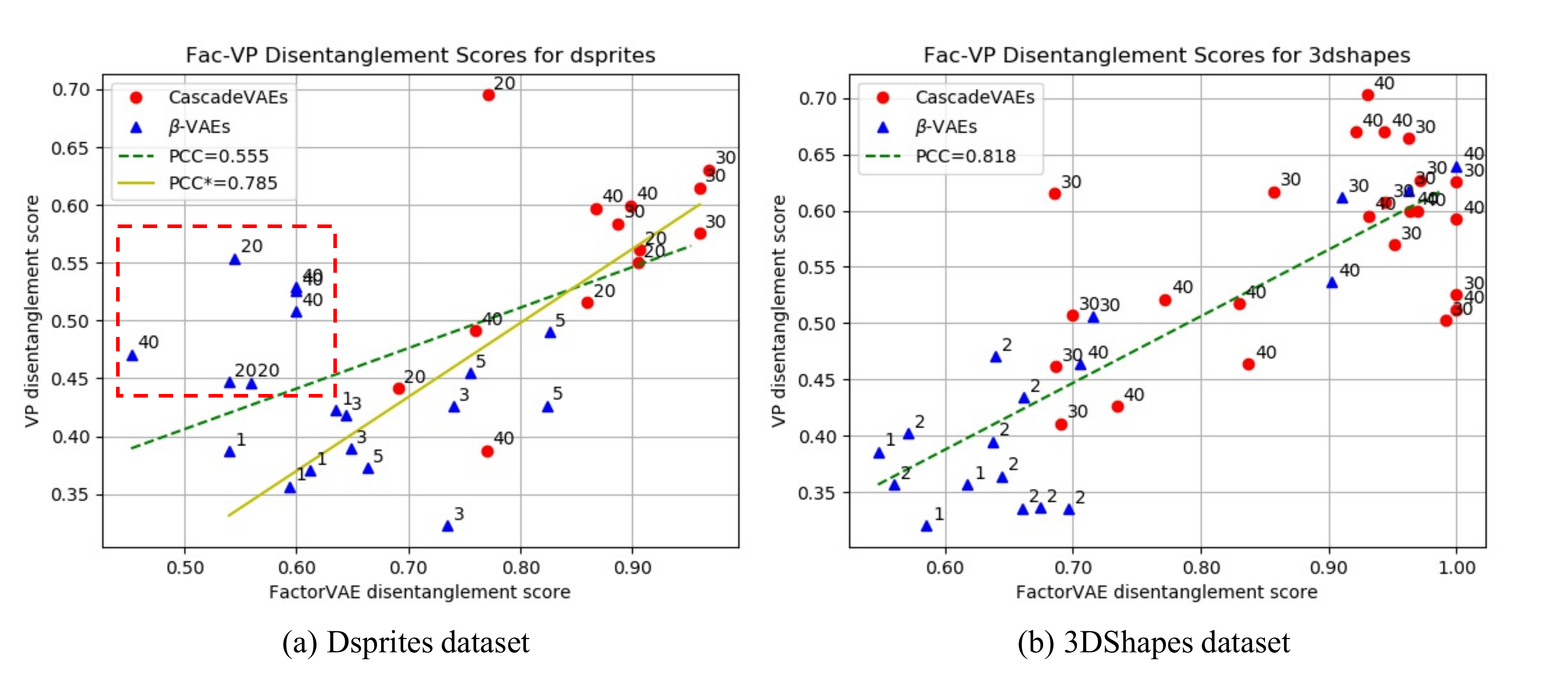}
\end{center}
    \caption{These are scatter plots shown between the FactorVAE
    disentanglement metric and our proposed VP disentanglement metric on
    Dsprites dataset and 3DShapes dataset. The $PCC$ denotes the
    Pearson's Correlation Coefficient, and the green dash line is the
    linear regression line. See the main text for discussion.
    }
\label{fig:correlation_show}
\end{figure}

\subsection{VP Disentanglement Evaluation Metric}
In this section, we evaluate the effectiveness of our proposed VP
disentanglement metric. Specifically, we reimplement
a relatively basic disentanglement model $\beta$-VAE
\cite{Burgess2018UnderstandingDI} and
a more advanced model CascadeVAE \cite{Jeong2019LearningDA},
and train them with different hyper-parameters and random seeds
to cover a large range of model performance.
Then we obtain their FactorVAE disentanglement metric scores,
a widely-used disentanglement measurement shown to
be correlated well with other disentanglement metrics and with
qualitative evaluation \cite{Locatello2018ChallengingCA,Kim2018DisentanglingBF}.
We then obtain the VP metric scores of the trained models
and see if these scores have correlation with
the ones calculated based on the FactorVAE metric.
Note that the FactorVAE metric requires the
existence of ground-truth generative factors of each data point in
the training dataset for
the performance measurement, while our VP metric does not use
the ground-truth factors.
This experiment is repeated on Dsprites dataset and 3DShapes dataset,
which are two most popular datasets for the
learning of disentangled representations.
The correlation results for both datasets are shown in Fig.
\ref{fig:correlation_show} (a) and (b) respectively.
The $\beta$ hyper-parameter in $\beta$-VAE is sampled from $\{1, 2, 3, 5,
20, 30, 40\}$,
and the hyper-parameters $\beta_{low}$ and $\beta_{high}$ in CascadeVAE
are sampled from $\{2, 3, 4\}$ and $\{20, 30, 40\}$ respectively. The models
are run with multiple random seeds.

From Fig. \ref{fig:correlation_show} (a) and (b), we can see there is
an evident correlation between our VP metric and the FactorVAE metric.
Note that the training of these models are not guided by the VP objective, and
there is no relation between the design of FactorVAE metric and
our proposed metric. Considering our metric requires no ground-truth factors
and the performances of these models suffer from the
impact of randomness during training,
this correlation is already very strong.
For Dsprites dataset, the Pearson's Correlation Coefficient
$0.555$ is not as
high as the one calculated on 3DShapes dataset $PCC=0.818$.
If we take a closer look, we can see the abnormal events happen
when $\beta$ is high in $\beta$-VAE models. This is because when $\beta$
is high (the ones in the red box in Fig. \ref{fig:correlation_show} (a)),
the model tends to ignore the \emph{shape} and \emph{rotation}
information in the Dsprites dataset while keeping \emph{xy-shift} and
\emph{scaling} disentangled, which leads to a low FactorVAE metric score.
However, this ignorance of information is not known
by our VP metric as it takes nothing from the ground-truth factors.
This causes our metric only take into account the other remaining factors
which are well-disentangled so high VP scores are obtained.
If we omit the high-$\beta$ models (in the red box), we get a
much higher $PCC=0.785$ that is close to the one obtained on 3DShapes dataset.

In summary, our proposed VP disentanglement metric correlates well with
FactorVAE disentanglement metric even though ours does not rely on the
ground-truth generative factors. This property makes our
metric a general evaluation method for all disentanglement models with
a generator, and it is applicable to datasets with or without
ground-truth generative factors. 
In Appendix Section 8, 
we conducts experiments to show why small 
$\eta$ is preferable in our proposed VP metric.

\begin{table}[t]
    \parbox{.49\linewidth}{
    \begin{center}
    \begin{tabular}{l|c|c}
    \hline\hline
        Model & VP Score & FacVAE Score \\
    \hline\hline
        CasVAE & 59.2 (4.6) & 91.3 (7.4) \\
        CasVAE-VP & \textbf{65.5} (5.1) & \textbf{91.7 (6.9)} \\
    \hline\hline
    \end{tabular}
    \end{center}
    \caption{Disentanglement scores on Dsprites dataset.}
    \label{table:dsprites_scores}
    }
    \parbox{.02\linewidth}{\ }
    \parbox{.49\linewidth}{
    \begin{center}
    \begin{tabular}{l|c|c}
    \hline\hline
        Model & VP Score & FacVAE Score \\
    \hline\hline
        CasVAE & 62.3 (4.9) & 94.7 (2.1) \\
        CasVAE-VP & \textbf{66.4} (5.6) & \textbf{95.6} (2.4) \\
    \hline\hline
    \end{tabular}
    \end{center}
    \caption{Disentanglement scores on 3DShapes dataset.}
    \label{table:3dshapes_scores}
    }
\end{table}

\subsection{VP Models with VAE Framework}
In this section we apply our VP models to popular disentanglement datasets
Dsprites and 3DShapes with VAE framework. We equip our VP loss to the
state-of-the-art disentanglement model CascadeVAE \cite{Jeong2019LearningDA}
to see if our variation predictability is complementary to the statistical
independence modeling and can boost disentangled representation learning.
The averaged disentanglement scores of 10 random seeds on two datasets
are shown in Table \ref{table:dsprites_scores} and Table
\ref{table:3dshapes_scores}. From the tables we can see our
model can boost the disentanglement performance of CascadeVAE, but the
performance improvement is not very significant.
We believe this is because on datasets
like Dsprites and 3DShapes the independence assumption is
crucial for the disentanglement learning, which has the most impact
on the learning of disentangled representations. However, this experiment
still shows our VP loss is complementary to the statistical independence
modeling, and is beneficial for disentanglement.
In Appendix Section 9, we show more 
quantitative comparisons on these two datasets.

\begin{table}[t]
    \parbox{.49\linewidth}{
    \begin{center}
    \begin{tabular}{l|c|c}
    \hline\hline
        Model & VP Score & FID \\
    \hline\hline
        VPGAN-flat $\alpha = 0.001$ & 58.4 & \textbf{22.3} \\
        VPGAN-flat $\alpha = 0.01$ & 62.1 & 27.8 \\
        VPGAN-flat $\alpha = 0.1$ & 64.5 & 32.8 \\
    \hline
        VPGAN-hierar $\alpha = 0.001$ & 64.6 & 53.6 \\
        VPGAN-hierar $\alpha = 0.01$ & 66.8 & 47.4 \\
        VPGAN-hierar $\alpha = 0.1$ & \textbf{70.3} & 56.9 \\
    \hline\hline
    \end{tabular}
    \end{center}
    \caption{Ablation studies of implementation and hyper-parameter $\alpha$
        on CelebA.}
    \label{table:ablation_celeba}
    }
    \parbox{.02\linewidth}{\ }
    \parbox{.49\linewidth}{
    \begin{center}
    \begin{tabular}{l|c|c}
    \hline\hline
        Model & VP Score & FID \\
    \hline\hline
        GAN & 12.9 & \textbf{20.4} \\
        InfoGAN $\lambda = 0.001$ & 34.5 & 24.1 \\
        InfoGAN $\lambda = 0.01$ & 25.3 & 43.3 \\
        FactorVAE & \textbf{75.0} & \textbf{\textcolor{red}{73.9}} \\
    \hline
        VPGAN-flat $\alpha = 0.1$ & 64.5 & 32.8 \\
        VPGAN-hierar $\alpha = 0.1$ & \textbf{70.3} & 56.9 \\
    \hline\hline
    \end{tabular}
    \end{center}
    \caption{Disentanglement models comparison on CelebA dataset.}
    \label{table:compare_celeba}
    }
\end{table}

\subsection{VP Models and Metric on CelebA with GAN Framework}
In this section, we apply our VP models and metric to
the challenging CelebA \cite{Liu2014DeepLF} human face dataset.
This dataset consists of over 200,000 images of cropped real-world
human faces of various poses, backgrounds and
facial expressions. We crop the center $128 \times 128$ area of the
images as input for all models.

We first conduct ablation studies on two factors that influence the
performance of disentanglement in our models: 1) hierarchical inputs and
2) the hyper-parameter $\alpha$. The hierarchical latent input
impact is a phenomenon discovered in
\cite{Karras2020ASG,Karras2019AnalyzingAI,Zhao2017LearningHF} that
when feeding the input latent codes into different layers of the
generator, the learned representations tend to capture different levels
of semantics.
We compare this hierarchical implementation with models with a
traditional flat-input implementation
to see its impact on disentanglement.
The number of network layers and the latent codes are kept the same.
In our experiments, we use the VP metric and FID \cite{Heusel2017GANsTB}
to give a quantitative evaluation on how these two factors
impact the disentanglement and the image quality. The results are
shown in Table \ref{table:ablation_celeba}. The latent traversals
of VPGAN-flat $\alpha = 0.1$, InfoGAN $\lambda=0.001$,
VPGAN-hierarchical $\alpha = 0.1$, and FactorVAE $\gamma = 6.4$ 
are shown in Fig. \ref{fig:traversal_1},
and Fig. \ref{fig:traversal_2}, 
and more latent traversals can
be found in Appendix 11, 12, and 13.

From Table \ref{table:ablation_celeba},
we can see the VP disentanglement score has a positive correlation with
the hyper-parameter $\alpha$.
We can also summarize that
the hierarchical inputs can boost the
disentanglement score by an evident margin, indicating the
hierarchical nature in deep neural networks can be a good ingredient
for the learning of disentangled representations.
On the other hand, the hyper-parameter $\alpha$ has a
slight negative impact on FID, 
therefore it is better to choose a
relatively small $\alpha$ to keep a balanced tradeoff between
disentanglement and image quality. Nevertheless,
the hierarchical input implementation
seems to have a more significant negative impact
on the FID, which we believe this technique should be
better used with larger number of latent codes and more
advanced architectures as in
\cite{Karras2020ASG,Karras2019AnalyzingAI} to
take full advantage of it.

From Table \ref{table:compare_celeba}, we can see our
VPGANs can achieve highest disentanglement scores
among GAN-based models and can even achieve close performance
as FactorVAE which models independence in the latent space.
However, the FactorVAE has a bad FID score, meaning the generated
images are lack of fidelity significantly.
On the contrary, our VPGANs can keep a better FID especially
the flat-input version.
We can see InfoGANs can achieve a certain level of disentanglement,
but their performance is significantly lower than VPGANs.
In practice, we also found the training
of InfoGANs are less stable than our VPGANs where InfoGANs may result
in generating all-black images, even though both types of models
are using the same generative and discriminative networks.
As a summary, our VPGANs keep a more balanced tradeoff between
the disentanglement and generation quality than the compared models.



\begin{figure}[t!]
\begin{center}
   \includegraphics[width=\linewidth]{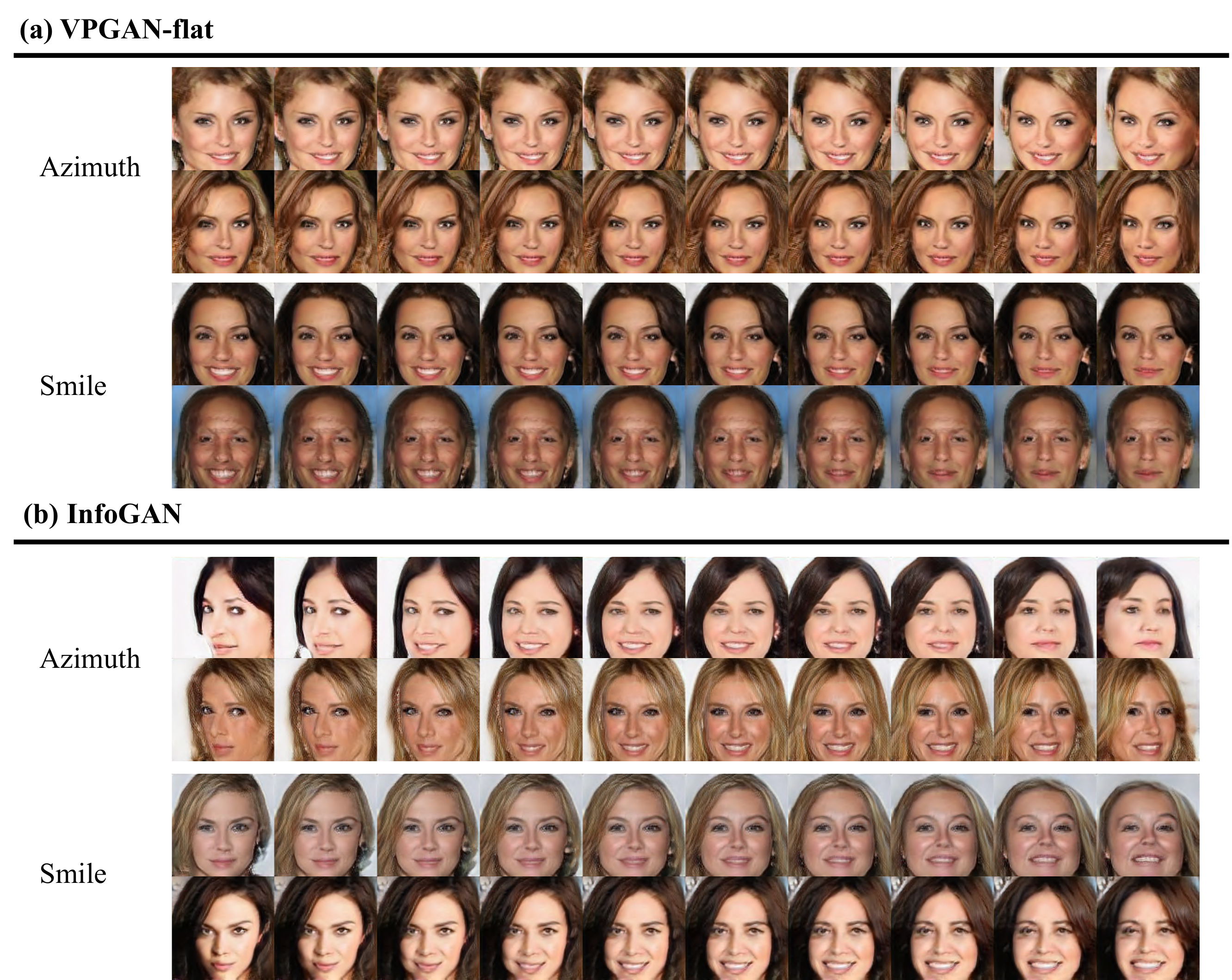}
\end{center}
    \caption{Latent traversals of VPGAN-flat and InfoGAN models
    on CelebA dataset. More latent traversals can be found in
    Appendix 11 and 12.}
\label{fig:traversal_1}
\end{figure}

\begin{figure}[t]
\begin{center}
   \includegraphics[width=\linewidth]{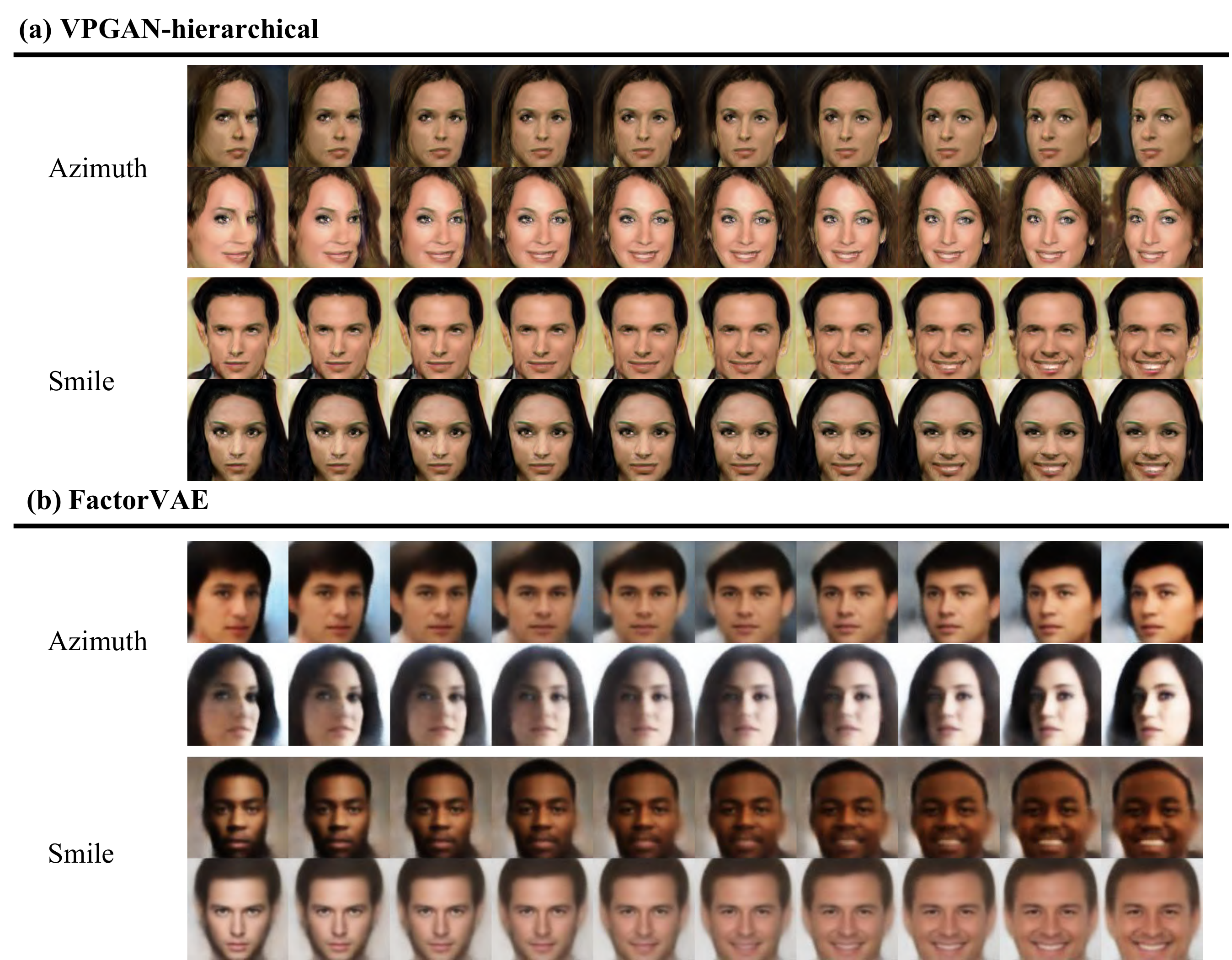}
\end{center}
    \caption{Latent traversals of VPGAN-hierarchical and FactorVAE models
    on CelebA dataset. More latent traversals can be found in
    Appendix 13.}
\label{fig:traversal_2}
\end{figure}

\begin{table}[t]
    \begin{minipage}{0.4\linewidth}
        \begin{center}
        \begin{tabular}{l|c|c}
        \hline\hline
            Model & VP Score & FID \\
        \hline\hline
            InfoGAN $\lambda = 0.1$ & 36.7 & 30.4 \\
            VPGAN $\alpha = 100$ & \textbf{42.0} & 32.1 \\
        \hline\hline
        \end{tabular}
        \end{center}
        \caption{Experiments on 3DChairs dataset comparing 
        \label{table:chairs_comp}
        InfoGAN and VPGAN.}
        \label{table:exp_chairs}
    \end{minipage}\hfill
    \begin{minipage}{0.55\linewidth}
        \begin{center}
        \includegraphics[width=\linewidth]{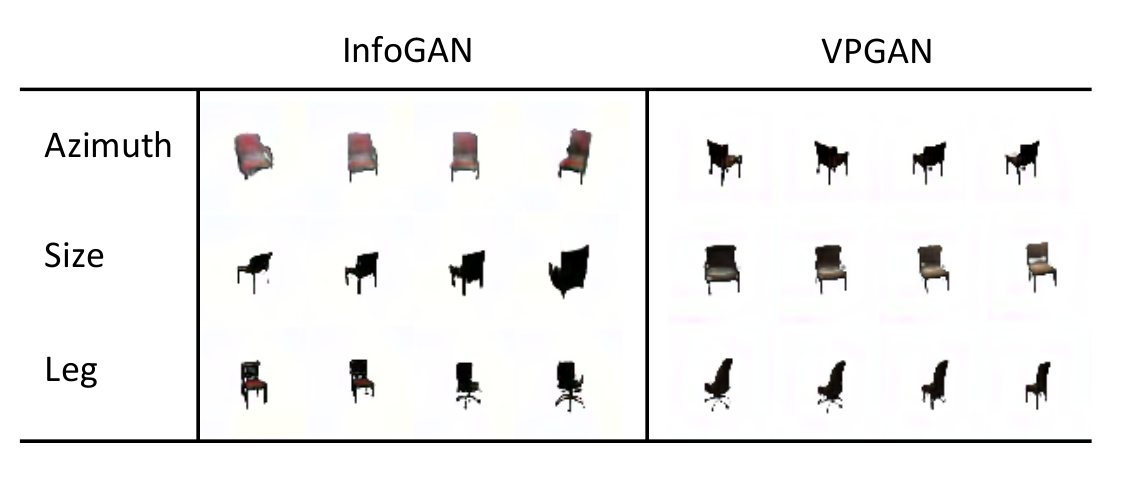}
        \captionof{figure}{Latent traversals on 3DChairs.}%
        \label{fig:chairs_traversal}
        \end{center}
    \end{minipage}
\end{table}

In Fig. \ref{fig:traversal_2} and Fig. \ref{fig:traversal_1} we
qualitatively show the performance of our VPGANs, 
InfoGAN and FactorVAE baselines in disentanglement by latent traversals.
From Fig. \ref{fig:traversal_1}, we can see our model learns a
cleaner semantics of azimuth while InfoGAN entangles azimuth with smile.
Our model also learns a better latent code for controlling smile, while
InfoGAN entangles smile and elevation into a single unit. 
In Appendix 11, 12 and 13, we show our VPGANs
can learn more semantics (azimuth, brightness, hair color, makeup,
smile, fringe, saturation, elevation, gender, lighting)
than InfoGANs (azimuth, brightness, hair color, fringe,
saturation, smile, gender, lighting). 
From Fig. \ref{fig:traversal_2}, we can see the FactorVAE entangles 
smile with some level of skin texture information, while 
our model achieves a cleaner disentanglement.
Also the results from FactorVAE are highly blurred, resulting in low FID.

There is an interesting phenomenon that for the learned disentangled 
representations in our models, not all dimensions encode variations.
We find there are around 1/3 of dimensions capturing no information 
(or too subtle to observe by eyes). Disentanglement prefers this property 
because it means the model does not randomly encode entangled information 
into the rest of the dimensions but instead deactivates them. 
When the number of latent factors is set to 25 - 30, the learning 
is stable and almost all semantics shown can be 
learned. For latent factors less than 15, we observe some semantics are 
absent or entangled.


\subsection{Experiments on 3D Chairs}
We compare InfoGAN and VPGAN on 3D Chairs dataset.
Quantitative results are shown in Table \ref{table:chairs_comp} 
and the latent traversals are shown in Fig. \ref{fig:chairs_traversal}. 
As we can see, VPGAN achieves a higher disentanglement score than 
InfoGAN at the cost of a slight increase in FID, which agrees 
with what we found in the CelebA experiments. From the traversals, 
our VPGAN learns a cleaner latent code on controlling azimuth 
semantics while InfoGAN entangles it with some shape information. 
However, the performance of VPGAN on this dataset is not as 
impressive as on CelebA, indicating a more delicate modeling 
than the variation predictability assumption is 
required for this dataset to achieve a perfect disentanglement.

\section{Conclusions}
In this paper, we introduced the latent \emph{Variation Predictability}
as a new proxy for learning disentangled representations.
By exploiting the latent variation predictability, 
we introduced the VP objective,
which maximizes the mutual information between the varied dimension
in the latent codes and the corresponding generated image pairs.
Apart from the VP objective, we also proposed a new evaluation metric
for quantitative measurement of disentanglement, which
exploits the prediction difficulty of the varied dimension in
the latent codes to quantify disentanglement. Different from
other disentanglement metrics, our proposed VP metric
does not require the existence of ground-truth generative factors.
Experiments confirm the effectiveness of our model and metric,
indicating the variation predictability can be exploited as a feasible
alternative to statistical independence for modeling disentanglement
in real-world data.
For future works, we aim to extend our work to downstream
applications like photorealistic image synthesis,
domain adaptation, and image editing.

\textbf{Acknowledgment}
This work was supported by Australian Research Council Projects 
FL-170100117, DP-180103424 and DE180101438. We thank Jiaxian Guo 
and Youjian Zhang for their constructive discussions.

\clearpage
%
%
\bibliographystyle{splncs04}
\bibliography{egbib}

\end{document}